\documentclass[letterpaper]{article}
\usepackage{ijcai13}

\usepackage{times}

\usepackage{amsmath}
\usepackage{amssymb}
\usepackage{graphicx}

\usepackage{url}
\usepackage{enumerate}

\usepackage{amsthm}

\theoremstyle{plain}
\newtheorem{theorem}{Theorem}

\theoremstyle{definition}

\theoremstyle{remark}
\newtheorem*{remark}{Remark}

\title{Duality in STRIPS planning} 
\author{Martin Suda \\
Max-Planck-Institut f\"ur Informatik, Saarbr\"ucken, Germany\\
Saarland University, Saarbr\"ucken, Germany\\
Charles University, Prague, Czech Republic}

\begin{document}

\maketitle


\begin{abstract}
We describe a duality mapping between STRIPS planning tasks.
By exchanging the initial and goal conditions, taking their
respective complements, and swapping for every action its
precondition and delete list, one obtains for every STRIPS
task its dual version, which has a solution if and only if
the original does. This is proved by showing that the described
transformation essentially turns progression (forward search) 
into regression (backward search) and vice versa.

The duality sheds new light on STRIPS planning by allowing a transfer
of ideas from one search approach to the other. It can be used to
construct new algorithms from old ones, or (equivalently) to obtain
new benchmarks from existing ones. Experiments show that the dual
versions of IPC benchmarks are in general quite difficult for modern planners.
This may be seen as a new challenge. On the other hand, the cases where the dual
versions are easier to solve demonstrate that the duality can also be made useful in practice.
\end{abstract}

\section*{Preamble}
The main theorem of this paper can be recovered already from \cite{massey_directions},
where it follows from a more general, but perhaps a less elegant result.
A proof similar to the one presented here can be found in \cite{pettersson_reversed}.


\section{Introduction}

Propositional STRIPS language is one of the favourite formalisms for describing planning tasks.
A STRIPS task description consists of an initial and goal condition formed by conjunctions of propositional atoms
and of a set of actions made up by a precondition, add and delete lists.
Despite its simplicity, the modelling power of the STRIPS formalism already captures the complexity class PSPACE \cite{bylander-strips}.
Also, STRIPS lies in the core of the more expressive PDDL language \cite{mcdermott-ipc-pddl} used
for representing benchmarks in the International Planning Competition.

Classical search is one of the basic but also most successful approaches to
determining whether a given planning task has a solution.
The search may proceed either in the forward direction
starting from the initial state and applying actions until a goal state is reached,
or in the backward direction where the goal condition is regressed over actions 
to produce sub-goals until a sub-goal satisfied by the initial state is obtained.
Forward search is typically termed progression,
while backward search is called regression.

In this paper we show that from the computational perspective
there is no real difference between progression and regression in STRIPS planning.
This is very surprising because progression is working with single states only
while the sub-goal conditions in regression represent whole state sets.
We show our result by describing a duality mapping working on the domain of all STRIPS planning tasks.
Performing regression on the original task is shown equivalent to performing progression on the dual.

The existence of the duality mapping has some additional interesting consequences.
For instance, any notion originally conceived and developed with one of the search approaches in mind
has a dual counterpart within the other approach. We give examples
of this phenomenon in Section~\ref{sec_dual}, one of them being the dual
of the relevance condition, an important ingredient in pruning the regression search space. 
The duality can also be used to construct new algorithms from old ones
and to obtain new benchmarks from existing ones.
Thus a purely theoretical concept at first sight,
the duality also has immediate implications for practice.

The rest of the paper is organized as follows. 
After giving the necessary preliminaries in Section~\ref{sec_prelim},
we recall the details about progression and regression relevant
for our work in Section~\ref{sec_prog_reg}.
Our duality mapping is defined and its properties are stated and proven in Section~\ref{sec_dual}.
We subsequently discuss immediate theoretical implications of the duality.
Section~\ref{sec_experiment} then reports on our experiments.
We compare the performance of several modern planners on dual versions
of IPC benchmarks and also show how a planner can be adapted 
with the help of the duality to solve benchmarks previously out of reach.
Finally, the concluding Section~\ref{sec_conclusion} discusses 
the applications of the duality from a broader perspective.


\section{Preliminaries}

\label{sec_prelim}

A propositional STRIPS \emph{planning task} is defined as a tuple $\mathcal{P} = (X,I,G,\mathcal{A})$, 
where $X$ is a finite set of \emph{atoms}, $I \subseteq X$ is the \emph{initial condition}, 	
$G \subseteq X$ is the \emph{goal condition}, and $\mathcal{A}$ a finite set of \emph{actions}.
Every action $a \in \mathcal{A}$ is a triple $a = (\mathit{pre}_a,\mathit{add}_a,\mathit{del}_a)$
of subsets of $X$ referred to as the action's \emph{precondition}, \emph{add list}, and \emph{delete list}, respectively.


The semantics is given by associating each planning task $\mathcal{P} = (X,I,G,\mathcal{A})$ with a \emph{transition system}
$\mathcal{T}_\mathcal{P} = (S,I,S_G,T)$, where the set of \emph{world states} $S = 2^X$ is identified with the set of all subsets of $X$,
the \emph{initial state} is the subset $I$, the \emph{goal states} $S_G = \{ s \in S \ |\ G \subseteq s \}$ are those states 
that satisfy the goal condition $G$, and, finally, the \emph{transition relation} $T$, 
which consists of state-action-state triples called \emph{transitions}, is defined as follows:
\[T = \{ s\stackrel{a}{\rightarrow}s' \ |\ \mathit{pre}_a\subseteq s \land s' = (s \cup \mathit{add}_a) \setminus \mathit{del}_a \} .\]
A planning task \emph{has a solution} if there is a \emph{path} in the respective transition system from the initial state to a goal state,
i.e. if there is a finite sequence of transitions $s_0 \stackrel{a_0}{\rightarrow} s_1 \stackrel{a_1}{\rightarrow} s_2 \ldots s_{k-1} \stackrel{a_{k-1}}{\rightarrow} s_k$ such that $s_0 = I$ and $s_k \in S_G$.


\section{Progression and regression}

\label{sec_prog_reg}


There are two basic approaches to searching for solutions of planning tasks:
progression and regression \cite{russel-norvig-ai-modern-approach}.
Progression, or simply forward search, 
proceeds systematically from the initial state and applies actions until a goal state is reached.
Regression, or backward search, on the other hand, 
regresses the goal condition over actions to produce sub-goals until a sub-goal contained in the initial state is obtained.

In what follows we abstract away the actual search algorithm
and only focus on properties of the two approaches that are important for showing their correctness.
These properties depend solely on  three ``entry point'' procedures,
by which the actual search algorithm could be parameterized:
\begin{description} 
\item
	$\texttt{start}()$, which generates a start \emph{search node},
\item
	$\texttt{is\_target}(t)$, which tests whether a given search node is a \emph{target} node, and
\item
	$\texttt{succ}(t)$, which generates  \emph{successor} nodes $t'$ of the given search node $t$.
\end{description}

\begin{table*}
\begin{center}
\begin{tabular}{l@{\ \ }|@{\ \ }c@{\ \ }|@{\ \ }c}
                        & progression: $\_^\mathit{Pr}$ & regression: $\_^\mathit{Re}$ \\
\hline
$\texttt{start}()$      & $I$                & $G$ \\
$\texttt{is-target}(t)$ & $G \subseteq t$    & $t \subseteq I$ \\
$\texttt{succ}(t)$      & $\{\ t' \ |\ \exists a \in \mathcal{A}\ .$\hspace{0.6cm} & $\{\ t' \ |\ \exists  a \in \mathcal{A}\ .$\hspace{0.6cm} \\           
                        & \hspace{1.0cm}$\mathit{pre}_a \subseteq t \ \land$ & \hspace{1.0cm}$ \mathit{del}_a \cap t = \emptyset \ \land$ \\    
                      & \hspace{0.9cm}$t' = (t \cup \mathit{add}_a) \setminus \mathit{del}_a \ \}$ & \hspace{0.9cm}$t' = (t \setminus \mathit{add}_a) \cup \mathit{pre}_a \ \}$ \\
\end{tabular}
\end{center}
\caption{Instantiating progression and regression for a plannig task $\mathcal{P} = (X,I,G,\mathcal{A})$. }
\label{table_progreg}
\end{table*}

Given a plannig task $\mathcal{P} = (X,I,G,\mathcal{A})$,
the respective implementations of the procedures for progression and regression are summarized in Table~\ref{table_progreg}.
Let us first focus on progression. There, each search node directly corresponds to a world state,
or, more specifically, to a world state reachable from the initial state.
The start search node $\texttt{start}^\mathit{Pr}()$ is equal to the initial state $I$ itself,
the $\texttt{is\_target}^\mathit{Pr}(t)$ procedure tests whether the given node satisfies the goal condition,
and the successor nodes $\texttt{succ}^\mathit{Pr}(t)$ are constructed by taking
for every action $a \in \mathcal{A}$ for which the \emph{applicability} condition $\mathit{pre}_a \subseteq t$ is satisfied 
the successor node $t' = (t \cup \mathit{add}_a) \setminus \mathit{del}_a$.
This naturally corresponds to the definition of the transition system $\mathcal{T}_\mathcal{P}$
and so the proof of the following correctness theorem for progression becomes immediate.

\begin{theorem} \label{thm_progrssion_correct}
A planning task $\mathcal{P} = (X,I,G,\mathcal{A})$ has a solution iff
there exists a sequence of search nodes $t_0, \ldots, t_k$ such that $t_0 = \texttt{start}^\mathit{Pr}()$, $\texttt{is\_target}^\mathit{Pr}(t_k)$,
and for every $i = 1,\ldots, k \ \ t_i \in \texttt{succ}^\mathit{Pr}(t_{i-1})$.
\end{theorem}

In the case of regression, a search node is also represented by a subset of $X$,
but it should be viewed as a sub-goal to be met, corresponding to a set of world states that satisfy it.
Here, the search nodes are manipulated in the following way. 
The start search node $\texttt{start}^\mathit{Re}()$ is identified with the (sub-)goal $G$ itself,
the $\texttt{is\_target}^\mathit{Re}(t)$ procedure returns true iff the initial state $I$ satisfies $t$,
and the successor search nodes $\texttt{succ}^\mathit{Re}(t)$ are generated
by collecting the regressed sub-goals $t' = (t \setminus \mathit{add}_a) \cup \mathit{pre}_a$
for every action $a \in \mathcal{A}$ for which the \emph{consistency} 
condition $\mathit{del}_a \cap t = \emptyset$ holds. The key property of regression is
that in every world state $s$ satisfying the regressed sub-goal $t'$ (i.e., in every $s$ such that $t' \subseteq s$)
the action $a$ is applicable ($\mathit{pre}_a \subseteq s$) and leads to a world state
that satisfies the original sub-goal $t$. Consistency is needed to ensure that the action doesn't undo 
any desired atom. 

\begin{remark}
Another property that is typically required, apart from consistency, 
is relevance.
An action $a \in \mathcal{A}$ is said to be \emph{relevant} 
for achieving a sub-goal $t$ iff $\mathit{add}_a \cap t \neq \emptyset$, i.e., if when applied, it achieves 
a part of the sub-goal. Because relevance is only important for efficiency and not for correctness of algorithms 
based on regression, we set it aside for now, to keep things simple, and return to it in a later discussion.
\end{remark}

The correctness theorem for regression has exactly the same form as the one for progression. 
We don't detail its proof, which is standard and basically just combines the insights mentioned above.

\begin{theorem} \label{thm_regrssion_correct}
A planning task $\mathcal{P} = (X,I,G,\mathcal{A})$ has a solution iff
there exists a sequence of search nodes $t_0, \ldots, t_k$ such that $t_0 = \texttt{start}^\mathit{Re}()$, $\texttt{is\_target}^\mathit{Re}(t_k)$,
and for every $i = 1,\ldots, k \ \ t_i \in \texttt{succ}^\mathit{Re}(t_{i-1})$.
\end{theorem}


\section{Duality}

\label{sec_dual}

When looking at Table~\ref{table_progreg}, which compares progression and regression, 
it is not difficult to observe certain formal similarities. For instance, the role played
by the initial condition $I$ in progression is similar to the one played by $G$ in regression and vice versa.
Similarly, the precondition $\mathit{pre}_a$ and delete list $\mathit{del}_a$
of the considered action $a$ seem to be exchanging roles in a certain way. 
In this section we describe an involutory mapping $\_^d : \mathit{STRIPS} \rightarrow \mathit{STRIPS}$
acting on the class of all STRIPS planning tasks that shows that the above similarities 
are not a coincidence and that progression and regression are more closely related than is would seem at first sight.

For an action $a = (\mathit{pre}_a,\mathit{add}_a,\mathit{del}_a)$ a dual action $a^d$ is
formed by exchanging the precondition and delete list: $a^d= (\mathit{del}_a,\mathit{add}_a,\mathit{pre}_a)$.
For a set of actions $\mathcal{A}$ the set of dual actions is $\mathcal{A}^d = \{ a^d \ |\ a \in \mathcal{A}\}$.
Now, given a planning task $\mathcal{P} = (X,I,G,\mathcal{A})$ the dual task $\mathcal{P}^d$ is obtained 
by exchanging the initial and goal conditions while taking their complements with respect to $X$, and using the dual action set:
\[ \mathcal{P}^d = (X,(X \setminus G),( X \setminus I),\mathcal{A}^d).\]


We can now state the central theorem of this paper.
\begin{theorem} \label{thm_duality}
For every planning task $\mathcal{P} = (X,I,G,\mathcal{A})$ 
the dual task $\mathcal{P}^d$ has a solution if and only if $\mathcal{P}$ does.
\end{theorem}
\begin{proof}
If a planning task has a solution, it can be found by both progression and regression, 
because they are both correct (Theorem~\ref{thm_progrssion_correct} and \ref{thm_regrssion_correct}).
We prove this theorem by showing that regression for $\mathcal{P}$ performs exactly the same operations
as progression for $\mathcal{P}^d$ when the search nodes are represented in a complemented form for the latter.
This is done in three steps corresponding to the three ``entry point'' procedures of Table~\ref{table_progreg}.
First, we realize that 
\[ \texttt{start}^\mathit{Re}_{\mathcal{P}}() = X \setminus \texttt{start}^\mathit{Pr}_{\mathcal{P}^d}(). \]
In words, the start search node of regression for $\mathcal{P}$, is the complement (with respect to $X$)
of the start search node of progression for $\mathcal{P}^d$. Similarly, a search node $t \subseteq X$ is 
a target node in regression for $\mathcal{P}$ if and only if $(X \setminus t)$ is a target node in progression for $\mathcal{P}^d$:
\[ \texttt{is\_target}^\mathit{Re}_{\mathcal{P}}(t) = \texttt{is\_target}^\mathit{Pr}_{\mathcal{P}^d}(X \setminus t), \]
which 
follows from the equivalence $a \subseteq b \leftrightarrow (X \setminus b) \subseteq (X \setminus a)$. 
Finally, the successor nodes of a search node $t \subseteq X$ in regression for $\mathcal{P}$ can be computed as 
complements of successor nodes of $(X \setminus t)$ in progression for $\mathcal{P}^d$:
\[ \texttt{succ}^\mathit{Re}_{\mathcal{P}}(t) = \{ (X \setminus t_0) \ |\ t_0 \in \texttt{succ}^\mathit{Pr}_{\mathcal{P}^d}( X \setminus t) \}. \]
For this last point, it is sufficient to verify for every action $a \in \mathcal{A}$ that, first, 
the consistency condition in regression for $\mathcal{P}$
and applicability condition in progression for $\mathcal{P}^d$ are each other's dual:
\begin{align*}
\mathit{del}_a \cap t = \emptyset \ &\leftrightarrow \ \mathit{del}_a \subseteq (X \setminus t) \\
                                    & \leftrightarrow \ \mathit{pre}_{a^d} \subseteq (X \setminus t) ,
\end{align*}
and that, second, regressing $t$ over $a$ yields the complement of applying $a^d$ to the complement of $t$:
\begin{align*}
X \setminus ((t \setminus \mathit{add}_a) \cup \mathit{pre}_a) &= ((X \setminus t) \cup \mathit{add}_a) \setminus \mathit{pre}_a \\
                                                               &= ((X \setminus t) \cup \mathit{add}_{a^d}) \setminus \mathit{del}_{a^d}.
\end{align*}
With these two properties checked (by applying De Morgan's laws for sets) the theorem is proven.
\end{proof}


The most striking consequence of Theorem~\ref{thm_duality} is the discovery that in STRIPS planning
there is no substantial difference between progression and regression.
Indeed, any algorithm based on one of the two approaches may be effectively
turned into an algorithm based on the other by simply applying the duality mapping to the input
as a preprocessing and running the actual algorithm on $\mathcal{P}^d$ instead of on $\mathcal{P}$.
It is then interesting to observe what are the dual counterparts of notions
that were originally conceived and developed with only one of the approaches in mind
and in how do they emerge ``on the other side of the duality''.
We will now comment on some of these observations in the following subsections.


\subsection{Relevance and usefulness}

It was mentioned before that it is important for the efficiency of regression
to only regress over actions that are relevant for the current sub-goal.
Let us repeat that an action $a \in \mathcal{A}$ is relevant for $t$
if and only if $\mathit{add}_a \cap t \neq \emptyset$. 
Regressing over an action that is not relevant for $t$
results in a (possibly strictly) stronger sub-goal $t' \supseteq t$.
We may safely discard $t'$ from consideration, 
because successfully regressing $t'$ is (possibly strictly)
more difficult than successfully regressing $t$.\footnote{If solution can be found from $t'$, it can be found from $t$ as well.} 
This way filtering out non-relevant actions helps to keep the regression search space manageable.

It is now at hand to ask what the dual notion of relevance is.
For lack of invention, we will call it usefulness.
We say that an action $a \in \mathcal{A}$ is \emph{useful} in a state $t$
if and only if the add list of $a$ is not fully contained in $t$.
We see that usefulness is a natural property:
it doesn't make sense to progress via a non-useful action,
because it will never make more atoms true in the resulting state.
The reason why usefulness is generally not mentioned in the literature
is that in typical benchmarks there are seldom actions that would be
applicable and yet not useful in a given state. This is in contrast
with regression where consistency and non-relevance are far less correlated.


\subsection{First add, then delete?}

When defining the result of action application to a state,
one needs to decide in which order should the add list and the delete list be considered.
In particular, if a description of a planning task contains an action $a$
such that $\mathit{add}_a$ and $\mathit{del}_a$ have a non-empty intersection,
the result of applying $a$ to a state $s$ depends on this order.
One can either exclude this possibility up front by requiring 
that for any action the add and delete lists are disjoint,
or, alternatively, to decide on a canonical order of their application. 

There are two remarks we can make here with respect to our duality.
First, if we choose the former option above, i.e., if we require 
that $\mathit{add}_a \cap \mathit{del}_a = \emptyset$ for any $a \in \mathcal{A}$,
we should perhaps (for the sake of symmetry) also require that 
$\mathit{add}_a \cap \mathit{pre}_a = \emptyset$, because that
is exactly the condition under which the order of applying add list 
and the precondition during regression of a sub-goal becomes irrelevant.
Note that this condition also makes sense from the perspective of progression,
because atoms mentioned in the precondition will be preserved by the action
anyway (unless deleted) so they don't need to be mentioned again in the add list.

The second remark relates to the latter option,
when to resolve the above situation a particular add-delete order is chosen as canonical.
Here the duality dictates (with appeal to elegance of the theory)
that adding should happen before deleting, as done in our definition in Section~\ref{sec_prelim}.
It is because only with that order the proof of Theorem~\ref{thm_duality} goes through as presented.
Let us be more specific. In progression we, quite naturally, 
first check the applicability condition $\mathit{pre}_a \subseteq s$,
before applying the effects. That's why the corresponding regression operation
needs to first subtract the add list from the sub-goal,
before adding the preconditions: $t' = (t \setminus \mathit{add}_a) \cup \mathit{pre}_a$.
Finally, dualizing the last equation gives us 
$s' = (s \cup \mathit{add}_a) \setminus \mathit{del}_a$ as promised.
This shouldn't be interpreted as saying that the duality itself
relies on a particular ordering of addition and deletion
in the definition of action application.
Should the other order be adopted instead, however, 
we would need to require that the actions 
of a planning task are normalized beforehand
so that the intersection of add and delete lists is always empty.


\subsection{Semantics of search nodes}

Since the duality exchanges the roles of progression and regression, 
one should ask what happens to the semantics of the search nodes,
which are known to represent world states in progression and 
sets of world states (via conjunctive conditions) in regression.
The surprising answer the duality gives us is that both the views 
are equally valid for both progression and regression.
One just needs to go over to the complement representation to see the other.
We invite the reader to check the details for herself
by replaying the proof of Theorem~\ref{thm_duality} from this perspective.
Note that this observation provides us with a new way 
(arguably less intuitive, but nevertheless a legitimate one)
to justify the correctness of the two approaches.
While this may sometimes simplify argumentations,
the actual implementation ``mechanics'' remains intact.










\subsection{Limitations}


We close this section by discussing the limitations of our duality concept.
A careful analysis of the proof of Theorem~\ref{thm_duality} reveals that 
it substantially relies on the particularly simple form of regression in STRIPS planning.
Essential is the fact that regressed sub-goals may be represented as conjunctions of atoms.
This means the duality doesn't directly carry over to more expressive formalisms, 
which allow negated goals or preconditions. For similar reasons,
extending the duality to Finite Domain Representation (FRD) \cite{helmert-fdr} seems problematic.
The good news is that the duality applies to the lifted version of STRIPS
as realized by the STRIPS subset of the PDDL language \cite{mcdermott-ipc-pddl} used 
in the International Planning Competition (IPC).\footnote{To complement the initial and goal condition,
one first obtains the set of all atoms $X$ by grounding the domain predicates
.} The IPC benchmark set contains more than a thousand practically relevant 
problems to which the duality applies.



\section{Experiment} 

\label{sec_experiment}

The duality mapping we have described in the previous section
provides us with a means of transforming one planning task into another
while preserving the existence of its solution. It is now natural
to ask how difficult are the dual versions of IPC benchmarks for modern planners.
We performed a series of experiments in order to answer this question
and we report on them in this section.

Note that there are two possible ways of interpreting the results.
We may either view the dual versions as new stand-alone problems,
or imagine the duality mapping as part of the algorithm we are currently testing. 
The second case may be understood as an evaluation of a new,
dual algorithm on the original benchmarks.
We will prefer the first view for most of this section,
but adopt the second where it is more natural.

For our experiments, we collected all the benchmarks from 
the satisficing tracks of the International Planning Competitions\footnote{\url{http://ipc.icaps-conference.org/}} 
that are in the STRIPS subset of the PDDL language.\footnote{We dropped the action cost feature where present.}
Together we collected 1564 problems. We then used the preprocessing part
of the planner FF \cite{hoffmann-nebel-ff-jair} to produce a grounded
version of these. Note that FF's relevance analysis was involved in the process,
so all the ``rigid'' predicates that are only used for modelling purposes and
the value of which is not affected by any action were removed.
Let us denote the set of these grounded IPC benchmarks \texttt{ORIG}.

The preprocessing tool was then extended further to implement our duality 
mapping: It first normalizes the actions so that the precondition and delete list
never intersect with the add list. To conform with the official IPC semantics, 
which is "first delete, then add" \cite{fox-long-pddl21},
this is done by performing for every action $a$ the following two assignments
in the prescribed order:
\[ \mathit{del}_a := \mathit{del}_a \setminus \mathit{add}_a; \ \ \ \mathit{add}_a := \mathit{add}_a \setminus \mathit{pre}_a.\]
Then the duality mapping is applied. Let the problems obtained this way
be denoted as \texttt{DUAL}.
All the experiments were performed our servers with 3.16~GHz Xeon CPU, 16~GB RAM, with Debian~6.0.


In the first experiment we ran the following three planners
on both \texttt{ORIG} and \texttt{DUAL} benchmark sets:
\begin{itemize}
\item
	the FF planner \cite{hoffmann-nebel-ff-jair}
	as a baseline representative of heuristic search \cite{bonet-geffner-heuristics} planners,
	
\item
	the LAMA planner \cite{richter-westphal-lama}, another heuristic search planner,
	the winner of the satisfycing track of the last IPC held in 2011, and

\item
	the planner Mp \cite{rintanen-mp}, as a representative of the planning as satisfiability \cite{kautz-selman-envelope} approach.
\end{itemize}
The time limit was set to 180 seconds per problem. 

The results of the first experiment are summarized in Table~\ref{table_results1}.
We see that the problems in \texttt{DUAL} are generally much more difficult to solve than \texttt{ORIG},
and that the SAT-based planner Mp seems to perform better on \texttt{DUAL} than the heuristic search planners.

\begin{table}
\begin{center}
\begin{tabular}{l|ccc}
                        & FF & LAMA & Mp \\
\hline
\texttt{ORIG} &
1009 & 
1192 & 
1114 \\
\texttt{DUAL} &
 136 &
 175 &
 329 \\
\end{tabular}
\end{center}
\caption{First experiment: number of \texttt{ORIG} and \texttt{DUAL} problems solved within 180 seconds by the respective planners.}
\label{table_results1}
\end{table}

We conjecture (and later partially verify) the following reasons for the difficulty of \texttt{DUAL}.
First, the explicit state forward search planners suffer from 
not testing for usefulness of actions. This corresponds to omitting 
the relevance test in the dual, regression-based algorithm
and makes the search space unnecessarily large.
The second reason is that \emph{invariant} information is no longer recovered from the task description by the planners.
Invariant is a property which holds in the initial state and is preserved by all transitions.
While logically redundant, invariants are known to be usually critical for efficiency of SAT-based planners \cite{rintanen-mp}.
Moreover, the existence of simple invariants formed by negative binary clauses is a prerequisite
for the reconstruction of a non-trivial Finite Domain Representation (FDR),
which LAMA is trying to build in its preprocessing phase \cite{helmert-fdr}.
As we independently checked, there are almost no binary clause invariants to be recovered
from the \texttt{DUAL} benchmarks. This means that Mp has to search for plans
without the useful guidance the invariants usually provide and LAMA most 
of the time discovers only trivial, two-valued domains for its finite domain variables.


\begin{remark}
Note that the problems in \texttt{DUAL} still contain the original invariant information,
but it has been turned into \emph{backward invariants}, 
properties of the goal states preserved when traversing the transitions backwards.
Obviously, the planners don't check for backward invariants,
because typically, e.g., on \texttt{ORIG}, it doesn't pay off.
\end{remark}


In our second experiment we set out to discover to what extent do the above reasons 
explain the degraded performance of the planners on \texttt{DUAL}. We focused
on the planner FF for its relative simplicity and modified it in several steps
in order to make it perform better on \texttt{DUAL}. We prepared the following 
versions of the planner:
\begin{itemize}
\item
	FF-U, which checks for usefulness of actions and discards the non-useful ones,
	
\item
	FF-UI, which additionally computes\footnote{We use an efficient implementation of the fixpoint algorithm described in \cite{rintanen-invariant}. }
	binary clause backward invariant,
	and discards successor states that violate it,

\item
	FF-UIN, which additionally turns off enforced hill climbing (see \cite{hoffmann-nebel-ff-jair})
	and always directly starts best first search.\footnote{We observed that enforced hill climbing
	fails on most of the problems in \texttt{DUAL}, so turning it off up front saves some time.}	
\end{itemize}
We ran all the modifications on \texttt{DUAL}, again with the time limit of 180 seconds per problem.

\begin{table}
\begin{center}
\begin{tabular}{l|cccc}
        & FF & FF-U & FF-UI & FF-UIN \\
\hline
\texttt{DUAL} &
   136 &
  204 &
 682 & 
 695 \\
\end{tabular}
\end{center}
\caption{Second experiment: number of problems from \texttt{DUAL} solved within 180 seconds by modifications of the planner FF.}
\label{table_results2}
\end{table}

The numbers of problems solved by the respective modifications are shown in Table~\ref{table_results2}.
For the sake of comparison we also repeat the result for the original FF. It can be seen that 
each of the modifications represents an improvement over the previous version. 
Probably the most is gained by incorporating the backward invariant test. Actually, 
each modification solves a strict superset of the problems solved by the previous one.
An exception is the last step where FF-UI solves 3 problems that FF-UIN cannot solve.
However, FF-UIN solves 16 problems that FF-UI cannot solve within the given time limit.


Despite our efforts to improve the performance of FF on \texttt{DUAL},
the planner still solves less problems from \texttt{DUAL} than from \texttt{ORIG}.
In our third experiment we tried to discover whether there are some problems in \texttt{DUAL}
that the improved FF-UIN can solve, while the original FF fails on their counterparts in \texttt{ORIG}.
This corresponds to the question whether our duality can be made useful in practice
by helping to solve difficult IPC benchmarks. To simplify the following discussion,
let us call by FF-DUAL a planner composed by the preprocessor, which grounds and dualizes inputs,
followed by FF-UIN. We will now compare FF and FF-DUAL on \texttt{ORIG}.

Apart from six problems from the Mystery domain, 
where FF-DUAL correctly discovers that no plan can exists while FF timeouts,
there are three domains where FF-DUAL seems to perform consistently better than FF.
Table~\ref{table_results3} reports on the number of problems solved,
categorized by the domains.

In order to better understand the success of FF-DUAL on the three domains,
we more closely analyzed and compared the output of the two versions of the planner.
In particular, we focused on the reported heuristic value of the currently expanded state.
We noticed the following facts.
\begin{itemize}

\item On the domain PSR the heuristic value of the initial state is quite low (between 1 and 10).
  This holds for both FF and FF-DUAL, but the value for FF-DUAL is typically one higher than that for FF.
  In other words, the dual version of relaxed plan heuristic seems to be more informative on PSR.
  
\item
	On Woodworking, the heuristic value of the initial state ranges from 5 up to about 70.
	FF-DUAL's values are typically not higher, but stay quite close to those of FF.
	
\item
	Although on Floortile FF's heuristic is more informed than FF-DUAL's, 
	FF's goal agenda mechanism seems to be making suboptimal decisions in decomposing the goal into sub-goals.
	On three problems where FF's enforced hill climbing fails within the time limit and the goal agenda is discarded,
	FF then successfully finds a plan with best first search. At the same time, FF-DUAL
	directly looks for a plan using best first search and its less informed heuristic.
\end{itemize}
On all the other domains FF-DUAL's heuristic value of the initial state
is typically much lower than the corresponding estimate of FF.
This seems to explain the general lower effectiveness of FF-DUAL on the \texttt{ORIG} benchmarks.

\begin{table}
\begin{center}
\begin{tabular}{lc|cccc}
                   &   & FF (unique) & FF-DUAL (unique) \\
\hline
PSR & (50) & 39 (2) & 45 \,(8)\, \\
Woodworking & (50) & 18 (2) & 44 (28) \\
Floortile  & (30) & \ 7\ \ (0) & 17 (10) \\
\end{tabular}
\end{center}
\caption{Third experiment: Comparing FF and FF-DUAL on three domains where the latter dominates the former.
Size of each domain and the number of problems uniquely solved by the respective planner are shown in parenthesis.}
\label{table_results3}
\end{table}

To sum up, in our experiments we have shown that the dual versions
of IPC benchmarks are in general much more difficult to solve by modern
planners than the originals. This can be partially remedied
by adapting a planner to make use of specific features
the dual benchmarks poses but which are usually missing in the standard ones.
Although the imagined dualizing planner FF-DUAL doesn't
beat the original FF in the overall number of solved problems,
there are certain domains where it indeed pays off to apply the duality mapping
before looking for a plan. This represents one possible application of the duality concept in practice.


\section{Conclusion}

\label{sec_conclusion}

In this paper we have described a duality mapping on the domain of all STRIPS planning tasks.
Its existence shows that computationally, there is no real difference 
between performing progression and regression as they are each other's dual.
Differences between the two that one can measure in practice follow from asymmetries (with respect to the mapping)
of the concrete benchmarks and are not inherent to the search paradigms themselves.
We believe that understanding these asymmetries and 
their influence on the efficiency of planning algorithms deserves further study.

Furthermore, we have pointed to several applications of the duality itself.
We have shown that new theoretical insights may be obtained 
by translating known notions via the mapping and analyzing the obtained duals.
For instance, there necessarily exists a ``precondition relaxation heuristic''
a dual of the famous delete relaxation heuristic.

Next we studied the dual versions of the standard IPC benchmarks
and discovered they are quite difficult to solve for modern planners.
One could argue that there is nothing interesting about difficult benchmarks in themselves
if they don't come from practical applications -- for instance,
random problems form the phase transition region (see \cite{rintanen-random}) seem to have this status.
We, however, don't think the dual IPC benchmarks fall into the same category.
After all, they still encode the same transition structures as the originals,
albeit in a non-obvious way. Therefore, we believe they should be considered
as an auxiliary test set by anyone attempting to develop a really versatile planner.

Finally, we explored the possibility of using the duality to design new algorithms.
A simple modification of the planner FF, which uses the duality, was shown 
to improve over the original system on several benchmark domains.
Note that this obvious schema of first dualizing the input and then running
a known algorithm is not the only option of how the duality can be used.
More sophisticated algorithms combining progression and regression
tied together by the duality can be imagined.






\bibliographystyle{named}
\bibliography{duality}

\end{document}